%% file: main.tex
\documentclass{article}

    \usepackage[preprint,nonatbib]{neurips_2024}

\usepackage[sort, numbers]{natbib}

\usepackage{amsthm, amsmath, amssymb}
\usepackage{fancyvrb}

\usepackage{caption} 
\usepackage{subcaption}
\usepackage{wrapfig}
\usepackage{microtype}
\usepackage{diagbox}
\usepackage{xfrac}
\usepackage{siunitx}
\usepackage{enumitem}

\usepackage{latexsym}
\usepackage[ruled, vlined, nofillcomment, linesnumbered, noline, noend]{algorithm2e}

\usepackage[english]{babel}
\usepackage{booktabs}
\usepackage{color}
\usepackage{graphicx}
\usepackage{verbatim}
\usepackage{xspace}
\usepackage{xcolor}

\usepackage[normalem]{ulem}  
\usepackage{newtxmath}
\SetCommentSty{textnormal}
\SetKwInOut{Output}{output}
\SetKwInOut{Input}{input}
\DontPrintSemicolon
\SetNoFillComment
\SetAlgoNoLine

\newtheorem{theorem}{Theorem}

\newtheorem{lemma}[theorem]{Lemma}

\newtheorem{definition}{Definition}

\newtheorem*{proposition*}{Proposition}
\newtheorem*{lemma*}{Lemma}
\newtheorem*{corollary*}{Corollary}

\newcommand{\Vb}{\widehat{V}}

\newcommand{\normf}[1]{\|#1\|_{\textrm{F}}}
\newcommand{\normfs}[1]{\|#1\|^2_{\textrm{F}}}

\newcommand{\defeq}{\triangleq}
\newcommand{\Rom}[1]{\uppercase\expandafter{\romannumeral #1\relax}}
\newcommand{\rom}[1]{\lowercase\expandafter{\romannumeral #1\relax}}
\newcommand{\conv}{\operatorname{conv}}
\newcommand{\dom}{\operatorname{dom}}
\newcommand{\ovex}[3]{\ensuremath{\overset{#1}{#2}\vphantom{#2}^{#3}}}

\DeclareMathOperator{\sign}{sign}
\DeclareMathOperator{\prox}{prox}

\newcommand{\RR}{\mathbb{R}}
\newcommand{\NN}{\mathbb{N}}

\DeclareMathOperator*{\argmin}{arg\, min}

\newcommand{\asso}{\textsc{Asso}\xspace}
\newcommand{\grecond}{\textsc{Grecond}\xspace}
\newcommand{\mebf}{\textsc{mebf}\xspace}
\newcommand{\elbmf}{\textsc{Elbmf}\xspace}
\newcommand{\elb}{\elbmf}
\newcommand{\zhang}{\textsc{Zhang}\xspace}
\newcommand{\felb}{\textsc{Felb}\xspace}

\newcommand{\felbmu}{\textsc{Felb}$^\textsc{mu}$\xspace}

\newcommand{\oururl}{\href{https://osf.io/dkq56/?view_only=b32690aedb96491d87ac58d987b5fbbd}{https://osf.io/dkq56/}}

\usepackage[pdftitle={Federated Binary Matrix Factorization using Proximal Optimization}, 
            pdfauthor={Sebastian Dalleiger, Jilles Vreeken, Michael Kamp}, 
            hidelinks]{hyperref}

\title{Federated Binary Matrix Factorization\\using Proximal Optimization}

\author{%
Sebastian Dalleiger\\
KTH Royal Institute of Technology\\
\texttt{sdall@kth.se}\\
\And
Jilles Vreeken\\
CISPA Helmholtz Center for Information Security\\
\texttt{jv@cispa.de}\\
\And
Michael Kamp\\
Institute for AI in Medicine, UK Essen and\\
Ruhr University Bochum and Monash University\\
\texttt{michael.kamp@uk-essen.de}\\
}

\begin{document}

\maketitle

\begin{abstract}
    Identifying informative components in binary data is an essential task in many research areas, including life sciences, social sciences, and recommendation systems.
    Boolean matrix factorization (BMF) is a family of methods that performs this task by efficiently factorizing the data.
    In real-world settings, the data is often distributed across stakeholders and required to stay private, prohibiting the straightforward application of BMF.
    To adapt BMF to this context, we approach the problem from a federated-learning perspective, while building on a state-of-the-art continuous binary matrix factorization relaxation to BMF that enables efficient gradient-based optimization.
    We propose to only share the relaxed component matrices, which are aggregated centrally using a proximal operator that regularizes for binary outcomes.
    We show the convergence of our federated proximal gradient descent algorithm and provide differential privacy guarantees.
    Our extensive empirical evaluation demonstrates that our algorithm outperforms, in terms of quality and efficacy, federation schemes of state-of-the-art BMF methods on a diverse set of real-world and synthetic data.
\end{abstract}

\section{Introduction}
\label{sec:introduction}
Discovering patterns and dependencies in distributed binary data sources is a common problem in many applications, such as cancer genomics~\citep{Liang:2020:Bem}, recommendation systems~\citep{Ignatov:2014:Boolean}, and neuroscience~\citep{Haddad:2018:Identifying}. 
In particular, this data is often distributed horizontally (i.e., the rows of the data matrix are split across hosts) and may not be pooled. 
While biopsies are performed in different hospitals, each location measures the expression of a common set of genes.
Privacy regulations mandate that those measurements may not be shared, thereby limiting the amount of data and available to traditional machine learning methods.
Federated learning~\citep{mcmahan2017communication}, however, enables learning from distributed datasets without disclosing sensitive data.

Although there are methods for Federated Non-negative Matrix Factorization~\citep{li2021federated}, they are designed for real-valued data and do not achieve interpretable results for binary data~\cite{Miettinen:2008:Discrete,dalleiger2022efficiently}.
Like their non-federated Non-negative Matrix Factorization (NMF)~\citep{Paatero:1994:Positive, Lee:1999:Learning, Lee:2000:Algorithms} counterparts and other factorizations like Singular Value Decomposition~\citep{Golub:1996:Matrix}, Principal Component Analysis~\citep{Golub:1996:Matrix}, Federated NMF does not achieve interpretable results for binary data~\cite{Miettinen:2008:Discrete}. 

For \emph{centralized data}, Boolean matrix factorization (BMF) seeks to approximate a Boolean target matrix $A\in\{0,1\}^{n\times m}$ by the Boolean product 
\[
A\approx \left[U\circ V\right]_{ij}=\bigvee_{l\in [k]}U_{il}V_{lk}\enspace 
\]
of two low-rank Boolean factor matrices~\citep{Miettinen:2008:Discrete},
$U\in\{0,1\}^{n\times k}$ (\emph{feature matrix}) and $V\in\{0,1\}^{k\times m}$ (\emph{coefficient matrix}).
Although there are myriad heuristics to approximate this NP-hard problem, doing so for \emph{distributed data} without sharing private information remains an open problem.
Directly applying federated learning paradigms to BMF would mean to factorize locally and then aggregate centrally. This requires a function that yield valid aggregations, 
such as, \emph{rounded average}, \emph{majority vote}, or \emph{logical \emph{\texttt{or}}}. 
We depict the impact of such na\"ive yet valid aggregations in Fig.~\ref{fig:invader}(a), which highlights that even the best combination of a local factorization algorithm and an aggregation scheme
---here, \asso~\citep{Miettinen:2008:Discrete} using logical \texttt{or}---%
leads to bad reconstructions on a toy example. 

Recently, it was shown that continuously relaxing BMF into a regularized \emph{binary} matrix factorization problem using linear (rather than Boolean) algebra and proximal gradients, yields a highly efficient, highly scalable, approach with state-of-the-art performance \citep{dalleiger2022efficiently}.
Taking advantage of this approach, we propose a novel federated proximal-gradient method, \felb, that centrally, yet privacy-consciously aggregates non-sensitive coefficients using a  proximal-averaging aggregation scheme. 
As illustrated in Figure~\ref{fig:invader}(b), \felb achieves a nearly perfect reconstruction on the toy example. 
We demonstrate that our approach converges for a strictly monotonically increasing regularization rate. 
In principle, parallelization \emph{via} \felb allows us to scale up BMF, even if the data is centralized, to address problems where gradients are large.
Moreover, we show that applying the Gaussian mechanism~\citep{balle2018improving} guarantees differential privacy, and we empirically validate that the utility remains high.
We show that \felb outperforms baselines derived via straightforward parallelization of state-of-the-art BMF methods on numerous datasets. 

In summary, our main contributions are as follows:
\begin{itemize}[noitemsep,topsep=0pt,parsep=0pt,partopsep=0pt]
    \item We present \felb, a novel federated proximal-gradient-descent for BMF.
    \item We improve BMF regularization with a new adaptive proximal approach with \felbmu.
    \item We prove convergence and differential privacy guarantees for \felb.
    \item We experimentally show that \felb and \felbmu factorizes distributed Boolean matrices efficiently and accurately. %
\end{itemize}

\begin{figure}[t]
    \centering
    \begin{subfigure}{0.4\textwidth}
    \includegraphics[width=\textwidth]{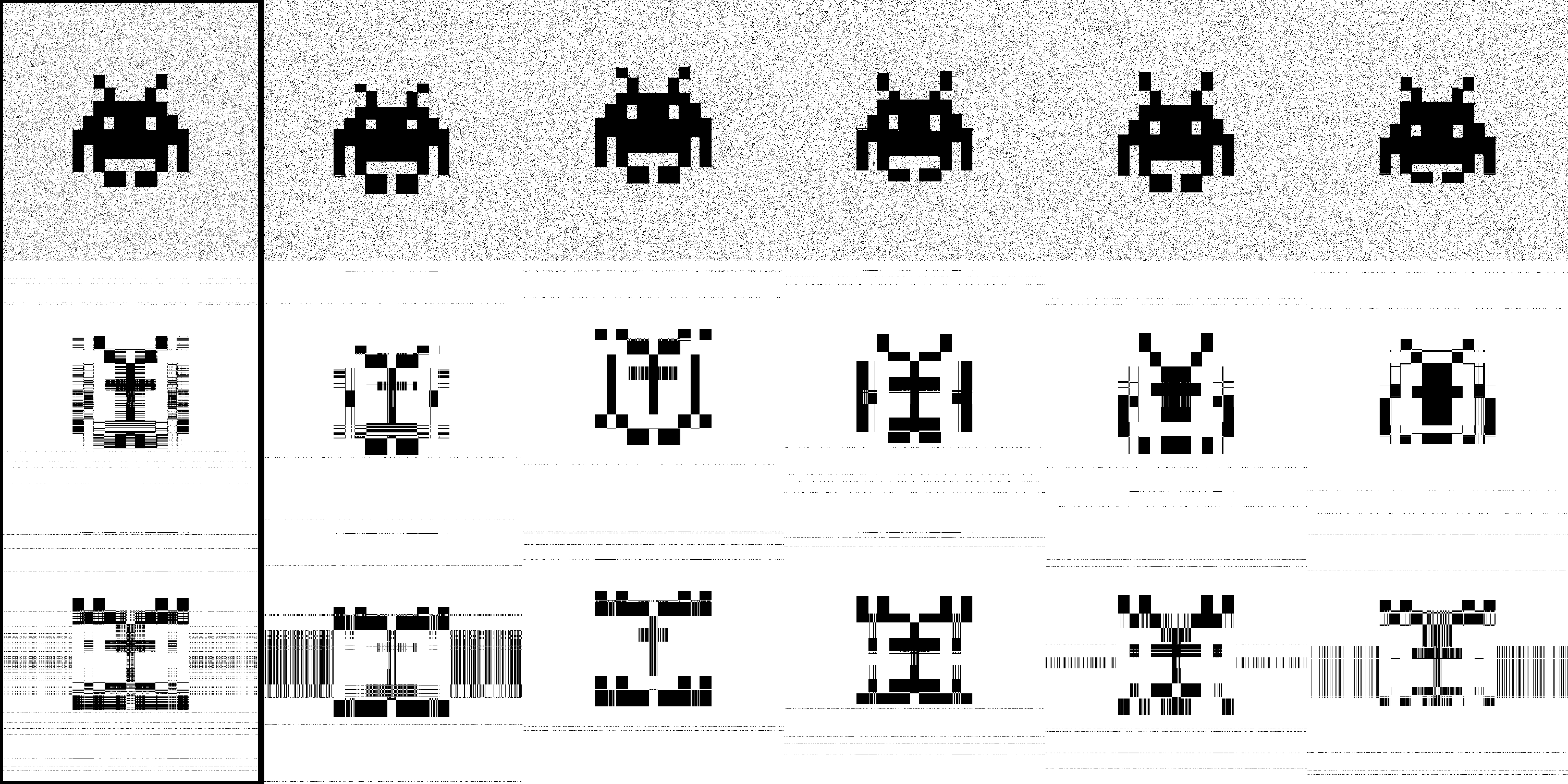}
    \caption{Aggregated \asso}
    \end{subfigure}\qquad
    \begin{subfigure}{0.4\textwidth}
    \includegraphics[width=\textwidth]{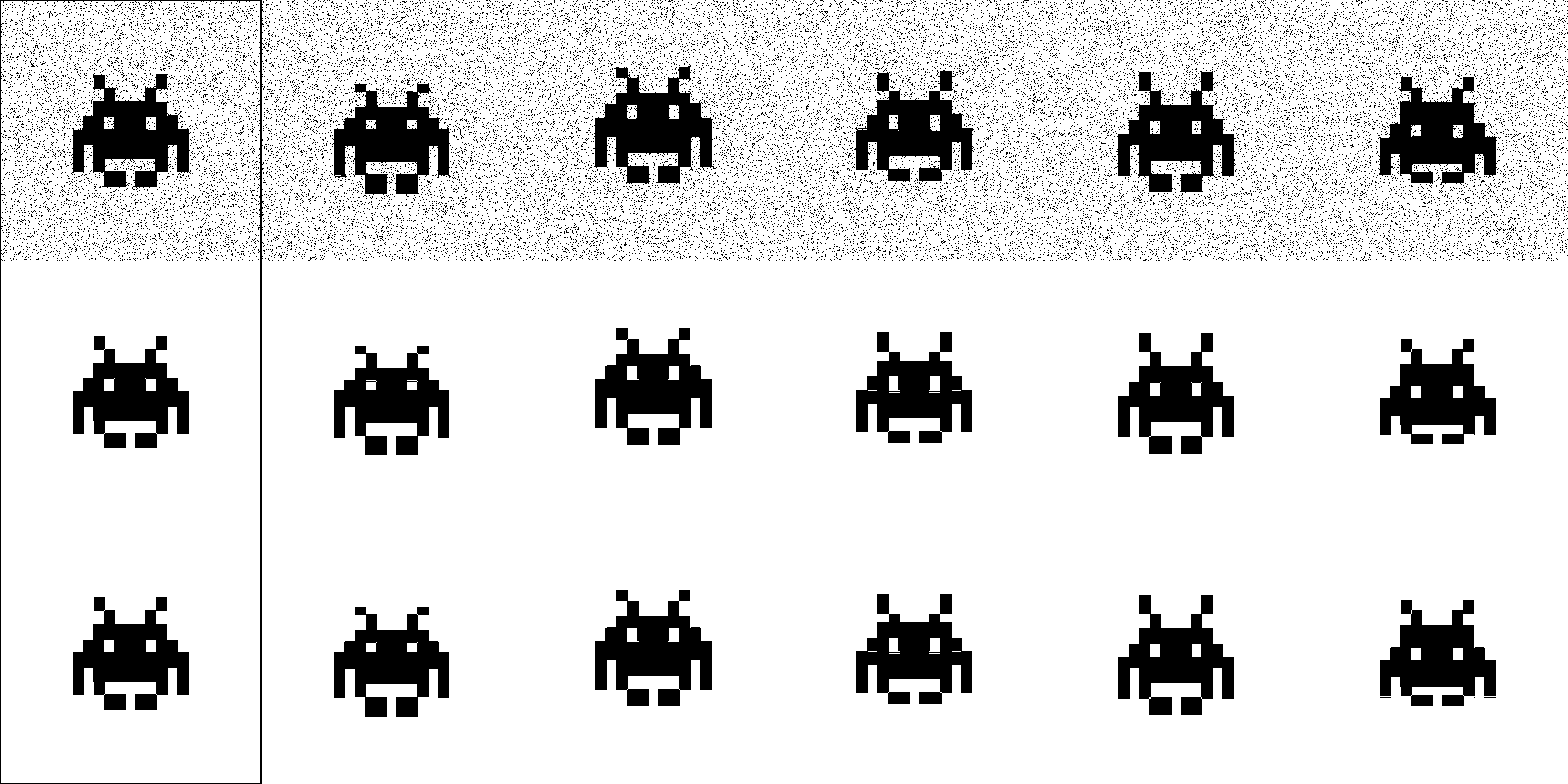}
    \caption{Proposed Method \felb}
    \end{subfigure}
    \caption{
    Our method reconstructs the input well. 
    Representing $1$s as black pixels, 
    for (a) \asso using 
    logical \texttt{or}
    and (b) our novel federated factorization called \felb,
    we show (top row) the client-data subjected to additive noise, 
    (middle row) the localized reconstructions, 
    and (bottom row) the aggregation-based reconstructions. 
    The left-most column shows centralized combination of the data resp. reconstructions of the five clients (columns 2--6).} 
    \label{fig:invader}
\end{figure}

\section{Related Work}
\label{sec:related}

To the best of our knowledge, there exists no federated BMF algorithms. 
We therefore primarily discuss the relations to \emph{BMF}, and \emph{federated factorization}, and  \emph{federated learning}.

We distinguish two classes of \textbf{BMF} methods: 
First, \emph{discrete optimization-based methods} that use Boolean algebra, such as \asso~\citep{Miettinen:2008:Discrete} using a set-cover-like approach,
\grecond~\citep{Belohlavek:2010:Discovery},
\mebf~\citep{Wan:2020:Fast} using fast geometric segmentation, or 
\textsc{Sofa}~\citep{Neumann:2020:Biclustering} based on streaming clustering.
Second, \emph{continuous optimization-based methods} that use linear algebra for solving the binary matrix factorization problem, introduced by \citet{Zhang:2007:Binary}, and advanced by 
\citet{Araujo:2016:Faststep} based on thresholding, 
and by 
Hess~et.~al~\citep{Hess:2017:Primping,Hess:2017:Csalt} using a proximal operator.
Combining ideas from the two complementary regularization strategies of \citet{Hess:2017:Primping} and \citet{Zhang:2007:Binary}, 
\citet{dalleiger2022efficiently} recently removed the need for post-processing via a proximal operator for an elastic-net-based regularizer.

With regards to \textbf{federated factorization} in general,
`parallel' algorithms for matrix factorization~\citep{yu2014parallel} as well as binary matrix factorization~\citep{khanna2013parallel} seek computational efficiency without addressing privacy concerns. 
The problem of matrix factorization for privacy-sensitive distributed data has been addressed by the federated-learning community with approaches for federated matrix factorization~\citep{du2021federated} and federated non-negative matrix factorization~\citep{li2021federated}.
These methods are, however, not specialized to Boolean matrices.
In this work, we seek to close the research gap, addressing the need for a federated, privacy-preserving binary (or Boolean) matrix factorization algorithm.

Recent advances in \textbf{federated learning} involve techniques like FedProx~\citep{li2020federated} and SCAFFOLD~\citep{karimireddy2020scaffold}. 
FedProx, an extension of FedAvg~\citep{mcmahan2017communication}, introduces a proximity penalty term to stabilizing the training process across different clients. 
SCAFFOLD enhances federated learning by correcting client drift using variance reduction techniques, thereby improving convergence rates and model accuracy compared to traditional methods like FedAvg,
while ProxSkip~\cite{Mishchenko:2022:ProxSkip} uses randomization to reduce the computational cost of proximal operators which are significantly more expensive than our operators.
Despite these advances, most research focuses on training deep neural networks using stochastic-gradient-based local optimization schemes. 
These approaches are often not ideal for factorizing matrices and are unsuitable for our case, as they neither incorporate constraint-penalties, nor do they handle alternating optimization problems, thereby achieving suboptimal empirical convergence towards infeasible non-Boolean solutions.

\section{Federated Proximal Binary Matrix Factorization}\label{sec:method}

Having contextualized our problem, we now formally introduce our federated Boolean matrix factorization scenario,
show how we separate our problem into manageable subproblems;
describe how to efficiently and solve subproblems in terms of binary matrix factorization relaxation, while preserving privacy;
and formally show that we compute a Boolean matrix factorization upon convergence.

The most pronounced difference between traditional and federated Boolean matrix factorization lies in data accessibility.
Rather than having all data $A\in\{0,1\}^{n\times m}$ accessible at one location, 
the data is given as (horizontally) partitioned matrices $A_i$ over $C\in\NN$ clients
\[
\exists A_1\in\{0,1\}^{n_1\times m},\dots,A_C\in\{0,1\}^{n_C\times m}: A=\begin{bmatrix}A_1,\cdots,A_C\end{bmatrix}^\top\enspace,
\]
such that $n = \sum_i n_i$.
We aim to discover a single \emph{shared} matrix $\widehat{V}\in\{0,1\}^{k\times m}$ containing shared feature components that are beneficial for all clients.
Due to privacy restrictions, we are however neither permitted to transmit matrices $A_i$ `offsite' (including to any other device), nor are we allowed {to be able} to draw conclusions about where components belong to.    
We want to factorize the data $A_i \approx U_i \circ \widehat{V}$
in terms of \emph{local} matrix $U_i\in\{0,1\}^{\frac{n}{C}\times k}$ (associating data to components), and one shared \emph{global} matrix $\widehat{V}\in\{0,1\}^{k\times m}$ (associating features into components). 
Without the knowledge of $U_i$, we \emph{cannot} estimate specific attributes of individual users (assuming sufficiently large client datasets). 
However, we \emph{can} estimate sets of commonly co-occurring attributes across all clients, such as common combinations of genetic markers that are indicative of a disease.

Locally computing $U_i$ for given $A_i$ and $\Vb$ is a regular Boolean matrix factorization.
However, computing the \emph{shared} $\Vb$ without access to $A_i$ and $U_i$ is not straightforward.
To enable the computing of a shared factor while still preserving privacy, we split the problem into subproblems $\Phi_i$, 
introducing a local \emph{but shareable} coefficient matrix $V_i\in\{0,1\}^{k\times m}$.
In a nutshell, we estimate a factorization for $\Phi_i$, combine local matrices $V_i\in\{0,1\}^{k\times m}$ into a shared matrix $\Vb$,
update $\Phi_i$, and repeat.
In a nutshell, we seek to optimize
\begin{equation}
    \argmin_{U, V, \Vb} \sum_i \Phi_i(U_i, V_i, \Vb) = \sum_i \normf{A_i - U_iV_i}  \;,
\end{equation}
specifying and solving the subproblems next.

\subsection{Local Subproblems and Clients}
A single subproblem at client $i \in \NN$, seeks to optimize $A_i \approx \left[U_i\circ V_i\right]_{ab}=\bigvee_{c\in [k]}U_{i,cl}V_{i,cb}$,
of two low-rank Boolean factor matrices~\citep{Miettinen:2008:Discrete}, $U_i\in\{0,1\}^{n_i\times k}$ (\emph{feature matrix}) and $V_i\in\{0,1\}^{k\times m}$ (\emph{coefficient matrix}).
As this problem is NP-complete~\citep{Miettinen:2008:Discrete},
solving it exactly is challenging for each client, even for relatively small matrices.
To solve this problem in practice even for large matrices, we resort to a continuous relaxation into a \emph{binary matrix factorization} problem that instead minimizes
\begin{equation}
    \normf{A_i-U_iV_i}^2+R(U_i)+R(V_i)\enspace ,
\end{equation}
for relaxed $U_i \in [0,1]^{n_i\times k}$ and $V_i \in [0,1]^{k\times m}$ 
with a \emph{binary-inducing regularizer} $R:\RR^{n'\times m'} \rightarrow \RR$,
enabling efficient gradient-based optimizations.
A regularizer that encourages binary solutions combines two elastic-nets (centered at $0$ and $1$, resp.),
\begin{equation}
    R_{\kappa\lambda}(X)=\sum_{x\in X}\min\left\{\kappa\|x\|_1+\lambda\|x\|_2^2,\kappa\|x-1\|_1+\lambda\|x-1\|_2^2\right\}
    \label{eq:elb_reg}
\end{equation}
into the almost W-shaped \textsc{Elb}-regularizer~\citep{dalleiger2022efficiently}.
To encourage that $V_i$ converges to $\Vb$, we introduce a proximity penalty term $P(V_i) = \gamma\normf{V_i - \Vb}^2$, yielding our global objective 
\begin{equation}
    \argmin_{U, V} \Phi_i(U_i, V_i) = \sum_i \normf{A_i - U_iV_i} + R(U_i) + R(V_i) + P(V_i)  \;.
\end{equation}
Even though now unconstrained, this problem is still challenging due to being non-convex.
We solve this joint objective by first splitting it in two subproblems, solving them alternating 
\begin{equation*}
    \begin{split}
        U_i^{t}=\argmin_{U} \normf{A_i-UV_i^{t-1}}^2 + R(U)\enspace\enspace\text{and}\enspace\enspace
        V_i^{t}=\argmin_{V} \normf{A_i-U_i^{t-1}V}^2 + R(V) + P(V)\;. \\
    \end{split}
\end{equation*}
Because each individual objective remains a challenge due to the non-convexity,
we require an optimization algorithm that is capable of solving such non-convex problems.
To this end, we employ the \emph{inertial proximal alternating linear minimization} (iPALM) technique~\cite{Pock:2016:Inertial}, which will guarantee convergence~\cite{Attouch:2013:Convergence,Bolte:2014:Proximal} as detailed in Sec.~\ref{sec:convergence}.

\paragraph{Proximal Alternating Linear Minimization}
At the core of iPALM, each regularized objective for $U_i$ and $V_i$ are solved using a proximal gradient approach,
which separates loss from regularizer.
That is, after taking a gradient step concerning our linear least-squares loss $f$, e.g., $f(U) \gets \normf{A_i-UV_i^{t-1}}^2$,
we then take a scaled proximal step regarding regularizer to project the gradient towards a feasible Boolean solution and towards a proximity to $\Vb$ for $V_i$.
A proximal operator is the projection
\begin{equation}
    \prox_{\eta}(X) = \argmin_{Y} \frac{\eta}{2}\normf{X - Y}^2 + R(X)
\end{equation}
of the result of the gradient step $x - x \eta \nabla_x f(x)$ for the loss $f$, 
into the proximity of a regularized solution $R(X)$. 
With regards to our regularizer $R$ and $P$, these proximal problems lend themselves for deriving first-order optimal and efficiently-computable closed-form solutions:
The \emph{Boolean proximal operator} for $R$ is element-wise computable $[ \prox_{\kappa\lambda}(X_{ik}) ]_{ij}$~\cite{dalleiger2022efficiently} where
\begin{equation}
    \prox^r_{\kappa\lambda}(x)=(1+\lambda)^{-1}\begin{cases}x-\kappa\sign(x)&\text{ if }x\leq\frac12\\
        x-\kappa\sign(x-1)+\lambda &\text{ otherwise}\end{cases}\enspace .
        \label{eq:elb_prox}
\end{equation}
The \emph{$\Vb$-proximity proximal operator} for $P$ is simply a weighted average
\begin{equation}
    \prox^p_{\gamma}(X) = [1+\gamma]^{-1} (X + \gamma \Vb) \;.
    \label{eq:prox_prox}
\end{equation}
Together, they yield the alternating update rules
\begin{align}
    U_i^{t+1}&=\prox^r_{\nu\kappa,\nu\lambda}(U_i^t-\nu\nabla^t_{U_i}\normf{A_i-U_i^tV_i^t}^2) \nonumber\\ 
    V_i^{t+1}&=\prox^p_{\xi\gamma} \prox^r_{\xi\kappa,\xi\lambda}(V_i^t-\xi\nabla_{V_i}\normf{A_i-U_i^tV_i^t}^2)\;,
    \label{eq:update}
\end{align}
for $\nu = \eta_{U_i}^t$ and $\xi = \eta^t_{V_i}$.
To apply these rules, we require step sizes, utilizing linear nature of the loss`',
we propose two alternatives: first we use the gradient Lipschitz constant $L$ for $\eta = \sfrac1L$, yielding the update rule for \felb. 
Second we employ \citet{Lee:2000:Algorithms}'s \emph{multiplicative update rule} (MU) for NMF with step size matrices $\eta^t_{U_i} = U_i \oslash U_iV_iV_i^{\top}$ and $\eta^t_{V_i} = V_i \oslash U_i^{\top}U_iV_i$ using the Hadamard division $\oslash$, containing individual step sizes for all elements in $U_i$ and $V_i$, yielding \felbmu.

\subsection{Global Objective and Server}
Now having established our per client subproblems,
we now combine the local subobjectives into one global objective
\begin{equation}
    \Phi(U, V, \Vb) = \sum_i \Phi_i(U_i, V_i, \Vb) = \sum_i \normf{A_i - U_iV_i} + R(U_i) + R(V_i) + R(\Vb) + P(V_i) \;,
\end{equation}
focusing on shared coefficients components $\Vb$.
To estimate the shared matrix $\Vb$ \emph{independent} of all data matrices $A_i$ and local basis matrices $U_i$, 
we have to combine $V_i$ matrices.
In federated learning, this is often done by aggregating all $V_i$ as the average $\Vb$.
However, doing so here does not necessarily yield valid results:
na\"ive averaging results in aggregates that are far from being binary, thus hindering or even preventing convergence. 
Addressing this aggregation problem, we aim to result in a Boolean matrix, for which we iteratively project the aggregate towards a valid Boolean values
\begin{equation}
    \Vb \gets \argmin_{\Vb} \sum_i \normf{\Vb - V_i} + R(\Vb) \;,
\end{equation}
for which we employ a proximal aggregation yielding the update-step %
\(
    \Vb \gets \prox^a_{\eta_{\Vb}\kappa,\eta_{\Vb}\lambda} \frac{1}{c}\sum V_i \;.
\)
To theoretically guarantee that privacy is preserved, clients may further distort the matrices $V_i$ before transmission,
thus ensuring \emph{differential privacy}, as described next.

\subsection{Guaranteeing Differential Privacy}
\label{sec:privacy}
The proposed aggregation approach only shares coefficient matrices, so that no direct relationships between observations are shared. 
An attacker or a curious server can, however, attempt to infer private data from coefficients $V_i$. 
Aiming to prevent this, we guarantee differential privacy using an additive noise mechanisms, where, in a nutshell, each client adds noise before it transmits $V_i$ to the server.
We consider the Bernoulli, Gaussian, and Laplacian mechanisms, which differ in the noise distribution. 
More formally, we achieve $(\epsilon,\delta)$-differential privacy using a Gaussian mechanism as follows.  
\begin{definition}[\citet{dwork2014algorithmic}]
For $\epsilon,\delta>0$, a randomized algorithm $\mathcal{A}:\mathcal{X}\rightarrow\mathcal{Y}$ is $(\epsilon,\delta)$-differentially private (DP) if 
\[
P\left(\mathcal{A}(X)\in S\right)\leq e^\epsilon P\left(\mathcal{A}(X')\in S\right)+\delta\enspace 
\]
holds for each subset $S\subset\mathcal{Y}$ and for all pairs of neighboring inputs $X,X'$ .
\end{definition}
Applying Gaussian noise with $0$ mean and $\sigma$ variance to the local coefficients $V_i$ before sending ensures $(\epsilon,\delta)$-DP for $\sigma=\Delta\epsilon^{-1}\sqrt{2\log(5/(4\delta))}$~\citep{balle2018improving}, where 
\(
\Delta=\sup_{X,X'}\|\mathcal{A}(X)-\mathcal{A}(X')\|
\)
is the sensitivity of $\mathcal{A}$. 
To ensure bounded sensitivity, we clip all $V_i$ with clipping threshold $\theta>1$~\citep{noble2022differentially}. 
Similarly, adding $0$-mean $\Delta\epsilon^{-1}$-variance Laplacian noise achieves $(\epsilon, 0)$-DP~\citep{dwork2006calibrating}.

\SetKwFor{local}{Locally}{do}{}
\SetKwFor{coord}{At server}{do}{}
\begin{algorithm}[t]
    \caption{Federated Binary Matrix Factorization with \felb}
    \small
    \label{alg:feddc}
    \KwIn{distributed target matrices $A^1,\dots,A^C$, component-count $k$}
    \KwOut{local feature matrices $U_1,\dots,U_C$, global coefficient matrix $V$}
    initialize $U_i,V_i$ for $i\in [C]$ uniformly at random\;
    \local{at client $i$ in iteration $t$}{
        $U_i \gets \prox^r_{\eta_{U_i}\kappa,\eta_{U_i}\lambda_t}( U_i - \eta_{U_i}\nabla_{U_i}\normf{A_i-U_iV_i}^2)$\;
        $V_i \gets \prox^r_{\eta_{V_i}\kappa,\eta_{V_i}\lambda_t} ( V_i - \eta_{V_i}\nabla_{V_i}\normf{A_i-U_iV_i}^2)$\;
        $V_i \gets \prox^p_{\eta_{V_i}\gamma}(V_i)$ \;
        \If{$t ~\textnormal{mod}~ b = 0$}{
            \If{is differentially private}{$V_i \gets V_i \oplus N, N\sim\mathcal{N}(0, \sigma)$\;}
            transmit $V_i$ to the server\;
            receive $\widehat{V}$ from the server\;
            let $V_i \gets \widehat{V}$\;
        }
    }
    \coord{}{
        receive $V_1,\dots,V_C$\;
        aggregate $\widehat{V}\leftarrow\prox^a_{\kappa\lambda}\left(\frac{1}{C}\sum_{i=1}^CV_i\right)$\;
        transmit $\widehat{V}$ to each client
    }
    \Return{$U$, $\widehat{V}$}\;
\end{algorithm}

\subsection{Convergence Analysis}
\label{sec:convergence}
Having ensured differential privacy, we summarize our algorithm.  %
We call the combination of this proximal aggregation with local proximal-gradient optimization steps the \felb algorithm, detailed in Alg.~\ref{alg:feddc}:  %
Local factors $U_i, V_i$ are initialized uniformly at random (line 1), and at each client in round $t$ (line 2), we update the local factor matrices (lines 4 and 5). 
Every $b$ rounds, we transmit the local matrices $V_i$ to the server (line 7). 
At this point, each client may choose to preserve differential privacy. 
The server, receives all local coefficients $V_i$ (line 11), averages the matrices, and applies the proximal-operator (line 12). 
The aggregate is then transmitted to all clients (line 13). 
Upon receiving the aggregate (lines 8 and 9), each client continues with the next optimization round.

Next, to formally ascertain that Alg.~\ref{alg:feddc} solves our problem, we show that the algorithm converges with Thm.~\ref{thm:global_convergence_main}, and achieves Boolean coefficients in the limit with Thm.~\ref{thm:bool_convergence_main}.

\begin{theorem}[Convergence]
    \label{thm:global_convergence_main}
    For the sequence generated by Alg.~\ref{alg:feddc} \(\{z^t \defeq (\{U^t_i\}_i, \{V^t_i\}_i, \bar{V}^t) \}_{k\in\NN}\),
    the objective function $\Phi(z^t)$ converges to a stable solution $\Phi(z^t) \to \widehat{\Phi}$ if $t \to \infty$.
\end{theorem}
\vspace*{-\baselineskip}
\begin{proof}(Sketch, full proof in Apx.~\ref{sec:apx:convergence}).~~%
    We show the objective's convergence to a stable solution $\Phi^\ast$ by initially establishing the convergence of each client, where we observe a \emph{sufficient reduction} in local objectives, as well as a \emph{bounded dissimilarity} to $\Vb$. 
    Leveraging this, we establish global convergence by showing that the global loss gradient is bounded by a \emph{diminishing term}, showing that $\Phi(z^{t})$ approaching a constant $\widehat{\Phi}$ as $t$ tends to infinity.
\end{proof}

\begin{theorem}[Boolean Convergence]
    \label{thm:bool_convergence_main}
    If $\lambda^t$ is a monotonically increasing sequence with $\lambda^{t-1} \leq \lambda^{t}$, $\lim \lambda^t \to \infty$, and $\lambda^{t} - \lambda^{t-1} \leq \infty$,
    then $V_1^T, \dots, V_c^T$ and $\Vb^T$ from the sequence generated by Alg.~\ref{alg:feddc} convergences as $\lim_{T \to \infty} \operatorname{dist}(\Vb^T, \{0,1\}) \to 0$ to a Boolean matrix.
\end{theorem}
\vspace*{-\baselineskip}
\begin{proof}(Sketch, full proof in Apx.~\ref{sec:apx:convergence}).~~%
    Since gradients are bounded and diminish, 
    we only need to show that the proximal operator returns Boolean solutions in the limit.
    As our gradients are Lipschitz continuous, bounded, and ensured to converge to a stable solution, 
    our scaled proximal operator projects values onto Boolean results,
    for a monotonically increasing regularizer rate $\lambda^t$ that approaches infinity in the limit, guaranteeing a stable Boolean convergence regardless of communication rounds.
\end{proof}

\section{Experiments}
\label{sec:experiments}

\paragraph{Competitors.} 
Given that there exist no {federated} matrix factorization algorithms tailored to binary data, we compare our approaches to {local} BMF methods, whose outcomes are then partially transmitted to a central location and collectively aggregated, 
following established ad-hoc federation strategies~\citep{kamp2019black}.
In particular, we adapt the localized algorithms, covering the state of the art in the method families (1)~\emph{cover-based Boolean matrix factorizations} (\asso,~\citet{Miettinen:2008:Discrete}; \grecond,~\citet{Belohlavek:2010:Discovery}; \mebf,~\citet{Wan:2020:Fast}) 
and (2)~\emph{relaxation-based binary matrix factorizations} (\zhang, \citet{Zhang:2007:Binary}; and \elb,~\citet{dalleiger2022efficiently}),
to factorize \emph{distributed} matrices---%
factorizing locally and aggregating the coefficient matrices centrally, replacing the local coefficients.
Leveraging the following aggregations, we summarize the BMF federation scheme in Apx.~\ref{sec:apx:competitors} Alg.~\ref{alg:aggregated_bmf}.
To ensure binary results, we employ three \emph{aggregation strategies} designed to maintain valid matrices
\begin{center}
\begin{tabular}{ccc}
    Rounded Average~~\refstepcounter{equation}(\theequation)\label{eq:aggmean} &
    Majority Vote~~\refstepcounter{equation}(\theequation)\label{eq:aggvote}   &
    Logical \texttt{Or}~~\refstepcounter{equation}(\theequation)\label{eq:aggor}      \\
    \(
        \left\lfloor C^{-1} \sum_{c \in [C]} V^c \right\rceil 
    \) &
    \(
        \left[\sum_{c \in [C]} V^c_{ij} \geq C/2\right]_{ij}
    \)&
    \(
        \bigvee V^1, \dots, V^C \;.
    \)
\end{tabular}
\end{center}

We now describe our diverse set of experimental setups.
First, we ascertain that \felb works reliably on synthetic data.
Second, we empirically assess the differential-privacy properties of \felb.
And third, we verify that \felb performs well on diverse real-world datasets drawn from four different scientific areas.
To quantify the results, we report the \emph{root mean squared deviation} (RMSD) and the F$_1$ \emph{score} between data and reconstruction,
as well as the runtime in seconds.

We implement \felb in the Julia language and run experiments on 32 CPU Cores of an AMD EPYC 7702 or one NVIDIA A40 GPU, reporting wall-clock time in seconds.
We provide the source code, datasets, synthetic dataset generator, and other information needed for reproducibility.
\!\footnote{Anonymized Repository: \oururl\ and Appendix \ref{sec:reproducibility}: Reproducibility}
In all experiments, we limit each algorithm run to 12h in total.
We quantify the performance of \emph{federated} \asso, \grecond, \elb, \mebf, \felb, and \felbmu in terms of loss, recall, similarity, and runtime, reporting results for \emph{majority voting} in the following, as it has superior performance to \emph{rounded averaging} and \emph{logical}, as shown in Apx.~\ref{sec:apx:aggregations}.

\subsection{Experiments on Synthetic Data}
In our experiments on synthetic data, we aim to answer the following questions:
\begin{itemize}[noitemsep,topsep=0pt,parsep=0pt,partopsep=0pt]
    \item[\bf Q1] How robust are the algorithms in the context of noise?
    \item[\bf Q2] How scalable are the algorithms with increasing client counts?
    \item[\bf Q3] How do the algorithms perform under differential privacy?
\end{itemize}
To answer these questions, we need a controlled test environment, which we construct by sampling random binomial-noise matrices into which we insert random densely populated`tiles' containing approximately $90\%$ with $1$s.
To highlight trends rather than random fluctuations, we report the mean and confidence intervals of $10$ randomized trials.  

\subsubsection{Robustness regarding Noise}\label{sec:noise}
To study the impact of noise on the quality of reconstructions, 
we generated synthetic matrices, introducing varying degrees of destructive XOR noise, ranging from $0\%$ (no noise, consisting solely of high-density tiles) to a maximum of $50\%$ (completely random noise). 
Employing a fixed number of $10$ clients, we applied federated \asso, \grecond, \mebf, \elb, and \zhang, alongside \felb and \felbmu to each dataset. 
We present RMSD, F$_1$ score (re signal and noise data), F$^\ast_1$ score (re signal), and runtime in Fig.~\ref{fig:noise}.

In Fig.~\ref{fig:noise}, we see that \felb and \felbmu achieve the best reconstructions across the board even at high noise levels.  
While the noise increases, the reconstruction quality declines across the board. 
We see that either RMSD and F$_1$ follows a similar trend across all method, while our methods consistly outperform the rest.
However, if we regard only the interesting data signal with F$^\ast_1$, we see that \felb and \felbmu are the only algorithms that still result in good reconstructions of the ground-truth signal even if the signal is hidden in high noise.
This shows the ability of \felb and \felbmu to discern pure noise from meaningful signal.
While the runtime of \asso, \grecond, \mebf, \zhang, and \elbmf is slightly faster in Fig.~\ref{fig:noise} (right), \felbmu's and \felb's runtime reduces with increasing noise levels. 

\begin{figure}[t]
    \centering
    \includegraphics[width=\linewidth]{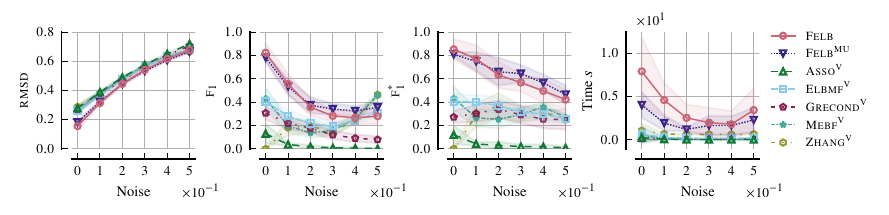}
    \caption{\felb and \felbmu are robust against noise.
    We show the loss, recall, similarity, and elapsed runtime ($s/C$) 
    for synthetic data with varying levels of destructive XOR noise.}
    \label{fig:noise}
\end{figure}

\subsubsection{Scalability regarding Clients}\label{sec:scalability}
Next, we analyze the scalability of federated \asso, \grecond, \elb, \mebf, and \zhang under \emph{majority voting}, as well as of \felb and \felbmu, for varying numbers of clients, considering two contrasting scenarios of \emph{scarce} and \emph{abundant} data. 
In both cases, we generate and uniformly distribute synthetic data to a number of clients, depicting results in Fig.~\ref{fig:scalability}.

To create \emph{data scarcity}, we fix the dataset size to $2^{16}$ and increase the number of clients from $2^2$ to $2^9$, thus iteratively reducing the sample count per client.
In Fig.~\ref{fig:scalability} (left), we observe that our methods scale well to low-sample scenarios and deliver the best performance. 
The MU update rule outperforms the competitors. 
The runtime of post-hoc federated methods \asso, \grecond, \mebf, \zhang, and \elb is lower since they only perform a single optimization epoch. 
These methods slightly outperform \felb and \felbmu only in tiny data scenarios where the estimator-variance is high, while the \felbmu significantly outperforms all methods and is notably faster than \felb.

To evaluate under \emph{data abundance}, we scale the number of samples by increasing the number of clients from $2^2$ to $2^9$, maintaining a constant sample count of $500$ per client. 
In Fig.~\ref{fig:scalability} (right), we observe that our methods scale well with an increased number of clients. With more data, \felb using Lipschitz steps slightly outperforms the MU steps in RMSD, and both methods exhibit comparable runtime trends. The runtime of post-hoc federated methods \asso, \grecond, \mebf, \zhang, and \elb remains lower, as they compute only one local optimization epoch.

\begin{figure}[tb]
    \includegraphics[width=\linewidth]{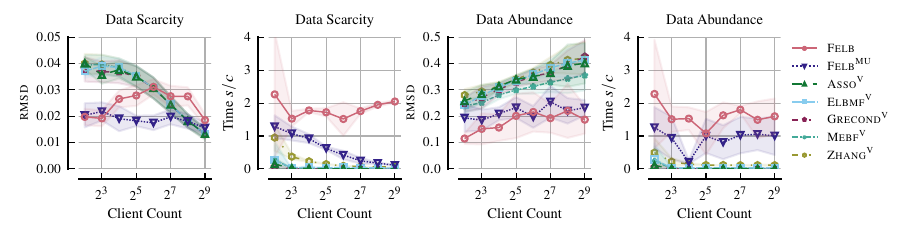}
    \caption{\felb and \felbmu perform well across various client counts, showing RMSD and runtime ($s/C$).
        For \emph{data scarcity}, we fix the data size and an increase number of clients.
        For \emph{data abundance} we grow data while increasing the number of clients.}
    \label{fig:scalability}
\end{figure}

\subsubsection{Performance under Privacy}\label{sec:exp:dp}
To empirically ascertain the effect of differential-privacy guarantees on the loss, we add noise to the transmitted factor matrices according to various noise mechanisms.
Specifically, we study the effect on algorithms subjected to additive clipped or regular {Laplacian} and {Gaussian}, as well as xor Bernoulli noise mechanisms, as depicted in Fig.~\ref{fig:dp} and Apx.~\ref{sec:apx:exp:privacy}, for varying $0 \leq \epsilon \leq 2$ and fixed $\delta = 0.05$.
Because \asso, \mebf, \grecond, \zhang, and \elbmf return Booelean matrices, we subject these only to xor noise, rather than additive noise, to retain Boolean matrices.
The results in Fig.~\ref{fig:dp}, show that both \felb and \felbmu exhibit similar performance across various noise models, while \felbmu is most robust. 
The plots display three phases:
In the low-$\epsilon$ domain, there is almost no performance deterioration, followed by a steep, hockey-stick-like descent which eventually stabilizes in the high-$\epsilon$ range. 
We note an increasing `sharpness' of the hockey-stick-phase under clipping, showing less smooth reactions to privacy adjustments for both mechanisms.

\begin{figure}[tb]
    \centering
    \includegraphics{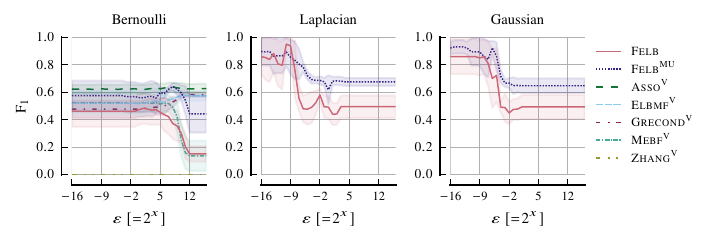}
    \caption{\felb and \felbmu achieve accurate yet differentially-private reconstructions. 
    For synthetic data, we subject algorithms to different noise mechanisms: Bernoulli, Laplacian, and Gaussian noise.}
    \label{fig:dp}
\end{figure}

\subsection{Experiments on Real-World Data}
Having established the efficiency and precision of our method in factorizing synthetic data, we proceed to assess its effectiveness in handling real-world datasets. For this purpose, we curated a diverse selection of 8 real-world datasets spanning four distinct domains.
To explore \textbf{recommendation systems}, we include \emph{Goodreads}\citep{Kotkov:2022:Goodreads} for books and \emph{Movielens}\citep{Harper:2016:Movielens} and \emph{Netflix}~\citep{NetflixPrize} for movies, where user ratings $\geq 3.5$ are binarized to $1$.
In \textbf{life sciences}, we use \emph{TCGA}\citep{TCGA} for cancer genomics, \emph{HPA}\citep{Bakken:2021:ComparativeCA,Sjöstedt:2020:hpa} for single-cell proteomics, and \emph{Genomics}~\citep{1kgc:2015:AGR} for mutation data. \emph{TCGA} marks gene expressions in the top 95\% quantile as $1$, while \emph{HPA} designates observed RNA in cells as $1$.
For \textbf{social science}, we analyze poverty \emph{(Pov)} and income \emph{(Inc)} using the \emph{ACS}~\citep{Census:2023} dataset, binarizing with one-hot encoding utilizing Folktables~\citep{ding:2021:retiring}.
In \textbf{natural language processing}, we study higher-order word co-occurrences in ArXiv cs.LG abstracts~\citep{Arxiv:2023}. 
Each paper abstract is a row with columns marked $1$ if the corresponding word is in the vocabulary, containing lemmatized, stop-word-free words with a minimum frequency of $1$ \textpertenthousand\ .
We summarize dataset extents, density, and chosen component counts in Apx.~\ref{sec:reproducibility}, Tbl.~\ref{tbl:data}.
For a fixed client count of $C=50$, we experimentally compare federated methods \asso, \grecond, \mebf, \elb, and \zhang, as well as \felb, and \felbmu across all real-world datasets, synchronizing after every $b=10$ local optimization rounds. 

In Tbl.~\ref{tab:realworld}, we present the F$_1$ between the reconstruction and the data matrix, where {}--{} indicate missing data due to time limits.
Our results see that \felb and \felbmu exhibit best-in-class performance, consistently ranking as the best or second-best algorithms. 
This performance gap is evident across all datasets except for the \emph{HPA} dataset, where \mebf, a method designed for similar data types, outperforms the others, and the \emph{ACS Pov} dataset, where \grecond leads. 
Notably, since clients of \elbmf and \zhang diverge significantly, they often aggregate into a no-consensus $0$-only global model matrix, thus showing low accuracy. 
Although they perform only a single optimization epoch per client, we see that \asso, \grecond, and \mebf do not finish on medium to large real-world datasets.
Additionally, we show the RMSD in Apx.~\ref{sec:apx:exp:realworld_rmsd}, where \felb and \felbmu are on top in RMSD, and compare client-server communication frequencies in Apx.~\ref{sec:apx:exp:realworld}, demonstrating the strength of \felb and resp.\ \felbmu.

\begin{table}
    \caption{\felb and \felbmu consistently achieve top performances.
    We illustrate the F$_1$ of \asso, \grecond, \mebf, \elb, and \zhang under voting aggregation, as well as federated \felb, and \felbmu
    on $8$ real-world data across $50$ clients.
    We highlight the best algorithm with \textbf{bold}, the second best with \underline{underline}, and indicate missing data due to timeouts by a dash {}--{}.}
    \label{tab:realworld}
    \centering
    \small
    \vspace*{0.5\baselineskip}
    \input{figures/115/f1.tbl.tex}

\end{table}

\section{Discussion and Conclusion}
\label{sec:discussion}

We introduced the federated proximal-gradient-based \felb for BMF tasks,
showed its convergence to a binary outcome in theory,
and demonstrated its efficacy in experimental practice. 
We provided a variant called \felbmu, whose practical performance outcompetes \felb on many real-world datasets, especially under rare synchronizations.
Although \felb and \felbmu perform consistently well, both are first-of-their-kind federated BMF algorithms. 
As such, they leave ample room for further research. \\
\textbf{Limitations}~%
Our research focuses on learning from private Boolean data generated by similar sources at a few research centers, thus we concentrate on suitable experiments and abstain from distant but related problems, such as learning with millions of heterogeneous clients. Further, we experimentally demonstrate the practical limitations of our methods extensively. We provide a more detailed discussion of limitations in Apx.~\ref{sec:limitations}.\\
In \textbf{future work}, we aim to extend our approaches to allow for heterogeneous clients and data distributions, adapting our methods to learn from varied data distributions and characteristics. 
Additionally, we plan to explore large-scale federations, drawing inspiration from frameworks like Scaffold~\cite{karimireddy2020scaffold} and FedProx~\cite{Li:2020:FedProx} for efficient client sampling and variance controlling. 
Furthermore, we intend to investigate personalized federated learning techniques to improve the reconstructions in case of varied data sources.
Finally, we plan to move beyond Boolean data and seek explore the potential of allowing partial sharing of a subset of the client components $V_i$ to allow for multi-source multi-modal federated learning to improve model performance and generality.

\vfill\pagebreak

\bibliographystyle{plainnat}

\input{main.bbl}
\appendix
\clearpage
\section*{Supplementary Material}

In this Appendix, we provide supplementary information  
\begin{itemize}[noitemsep,topsep=0pt,parsep=0pt,partopsep=0pt]
    \item regarding the convergece in Apx.~\ref{sec:apx:convergence},
    \item regarding the federation of baseline BMF methods in Apx.~\ref{sec:apx:competitors},
    \item regarding dataset used in our experiments Apx.~\ref{sec:datasets},
    \item regarding reproducibility of our experiments in Apx.~\ref{sec:reproducibility},
    \item regarding limitations in Apx.~\ref{sec:limitations}.
\end{itemize}  
Furthermore, we provide additional experimental results regarding
\begin{itemize}[noitemsep,topsep=0pt,parsep=0pt,partopsep=0pt]
    \item post-hoc aggregations in Apx.~\ref{sec:apx:aggregations},
    \item empirical convergence in Apx.~\ref{sec:apx:exp:convergece},
    \item client drift in Apx.~\ref{sec:apx:exp:drift} 
    \item differential privacy in Apx.~\ref{sec:apx:exp:privacy}, and
    \item additional real-world performance evaluations in Apx.~\ref{sec:apx:exp:realworld}. 
\end{itemize}

\section{Convergence}
\label{sec:apx:convergence}
In this section, we establish the convergence properties of Algorithm \ref{alg:feddc}.
We begin by showing in Theorem \ref{thm:global_convergence} that the objective function of the algorithm converges to a stable solution in the limit.
To this end, we leverage the local convergence of each client, as proven with Lem.~\ref{thm:local_convergence} (Apx.~\ref{appdx:convergingclients}), to demonstrate a sufficient reduction in the \emph{global} objective function values.
By combining these results, we establish the global convergence of the objective function.
Building upon this, we moreover prove in Theorem \ref{thm:bool_convergence} that the algorithm converges to Boolean matrices.
We establish conditions under which the sequences of matrices converge to binary solutions, demonstrating that both the gradient and proximal operator converge to binary solutions, thereby ensuring the stability of Boolean solutions at both the global and local levels.
The outline of our proof is as follows.
\begin{enumerate}
    \item We show the convergence of Alg.~\ref{alg:feddc} in Thm.~\ref{thm:global_convergence}.
    \item We show that Alg.~\ref{alg:feddc} converges to Boolean matrices with Thm.~\ref{thm:bool_convergence}.
    \item We show the convergence of each client in Alg.~\ref{alg:feddc} to a stable solution with Lem.~\ref{thm:local_convergence}.
\end{enumerate}

\begin{theorem}[Convergence of Alg.~\ref{alg:feddc} (restated)]
    \label{thm:global_convergence}
    For the sequence generated by Alg.~\ref{alg:feddc} \(\{z^t \defeq (\{U^t_i\}_i, \{V^t_i\}_i, \bar{V}^t) \}_{k\in\NN}\),
    the objective function $\Phi(z^t)$
    converges to a stable solution $\Phi(z^t) \to \widehat{\Phi}$ if $t \to \infty$.
\end{theorem}

\begin{proof}
    To show that the objective convergence to a stable solution $\Phi(z^t) \to \Phi^\ast$ when $t \to \infty$,
    we first show that each client convergence in Lem.~\ref{thm:local_convergence},
    where we observe a sufficient reduction in $\Phi_i(z_i^{t+1}) \leq \Phi_i(z_i^{t+1}) - \rho_i \normfs{z_i^{t+1} - z^{t}}$ for some constant $\rho_i$.
    Using this property we can show the global convergence as follows.
    \begin{align*}
        \Phi(z^{t+1}) & = \sum_i \Phi_i(z_i^{t+1})                                                    \\
                      & \leq \sum_i \Phi_i(z_i^{t}) - \rho_i \normfs{\nabla_i \Phi_{z_i^{t}}(z_i^{t})} \\
                      & \leq \Phi(z^{t}) - \sum_i \rho_i \normfs{z_i^{t+1} - z^{t}}                    \\
                      & \leq \Phi(z^{t}) - \rho \sum_i \normfs{z_i^{t+1} - z^{t}}
    \end{align*}
    Moreover, from Lem.~\ref{thm:local_convergence} we deduce that
    $ \normfs{z_i^{t+1} - z_i^{t}} \to 0$ if $t \to \infty$.
    Therefore, the global loss converges, $\Phi(z^{t+1}) \to \widehat{\Phi}$ to some constant $\widehat{\Phi}$
\end{proof}

So far, we only know that our algorithm generates a convergent sequence.
It remains to show that the sequence converges to a Boolean solution, which follows in Thm.~\ref{thm:bool_convergence}.

\begin{theorem}[Boolean Convergence of Alg.~\ref{alg:feddc} (restated)]
    \label{thm:bool_convergence}
    If $\lambda^t$ is a monotonically increasing sequence with $\lambda^{t-1} \leq \lambda^{t}$, $\lim \lambda^t \to \infty$, and $\lambda^{t} - \lambda^{t-1} \leq \infty$,
    then $V_1^T, \cdots, V_c^T$ and $\Vb^T$ from the sequence generated by Alg.~\ref{alg:feddc} convergences as $\lim_{T \to \infty} \operatorname{dist}(\Vb^T, \{0,1\}) \to 0$ to a Boolean matrix.
\end{theorem}
\begin{proof}
    In each update round, the client $i$ performs the proximal alternating linear minimization steps laid out in Eq.~\ref{eq:update}, yielding an updated $V_i^t$ (resp. $U_i^t$). %
    Focusing on $V_i$ (independent of client-server communication), we first show that the gradient of $V_i$ goes to zero.
    As shown by Thm.~\ref{thm:global_convergence} and Lem.~\ref{thm:local_convergence},
    our sequence of alternating linear optimization steps followed by scaled proximal steps convergence.
    Note that our gradients are bounded and are Lipschitz continuous.
    Because we scale our proximal operators with respect the Lipschitz moduli of the respective gradients,
    notably prevent the proximal operator and gradient steps from alternatingly between $0$ and $1$, thus creating a convergent sequence to a stable solution.
    We need to verify that the proximal operator projects to binary solutions, i.e.,
    $\lim_{\lambda^t\rightarrow\infty}\prox(x) \in \{0, 1\}$ for $\lambda^t \rightarrow \infty$.
    We do this with a case distinction:
    For $x \leq 0.5$, we obtain
    \(
    \lim (x-\kappa\sign(x)) (1+\lambda^t)^{-1} = 0 \;
    \),
    and analogously for $x > 0.5$, we obtain
    \(
    \lim (x-\kappa\sign(x-1)+\lambda^t) (1+\lambda^t)^{-1} = 1 \;,
    \)
    thus having ensured a binary proximity,
    for $\lambda^t \rightarrow \infty$ with $\lambda^{t} \leq \lambda^{t+1}$ and $\lambda^{t+1} - \lambda^{t} \leq \infty$, any bounded $x$, and finite $\kappa \in \RR_+$.
    Therefore, independent of communication rounds, the gradient converges to $0$ and the proximal operator converges to a binary solution.
    It remains to show that for $t\rightarrow\infty$, a binary solution stays stable, meaning that a global binary solution implies local convergence.
    By assuming that a client in round $t$ receives a binary aggregate $\widehat{V}$ from the server, we obtain $\|\eta\nabla_V \|A_i-U_i^{t-1}V_i\|_{max}\|\leq \epsilon$ for $\epsilon<\sfrac12$.
    By abbreviating the gradient-step result
    \[
        V' = V_i^{t-1} - \eta \nabla_{V_i} \normf{A_i-U_i^{t-1}V_i}^2
    \]
    we see that $V'_{pq}<\sfrac12$ if $[V_i^{t-1}]_{pq}=0$, and $V'_{pq}>\sfrac12$ if $[V_i^{t-1}]_{pq}=1$, which implies that $\prox_{\lambda^t\kappa}(V')$ is binary and $V_i^t=V_i^{t-1}$.
    Moreover, repeating these steps for $\Vb^t$, we obtain boolean aggregates upon convergence.
\end{proof}

\subsection{Converging Clients}
\label{appdx:convergingclients}
In this part, we demonstrate the convergence of each client in Algorithm \ref{alg:feddc}.
Specifically, we show that the decrease between client iterations is sufficiently large, while ensuring convergence to stable solutions.
To achieve this, we employ a series of lemmas:
In Lemma \ref{thm:local_convergence}, we establish that the sequence generated by each client converges both in terms of objective function value and to a critical point of the objective function.
We further provide that the difference of the sequence under finite length conditions is bounded.
Subsequently, Lemma \ref{lem:bounded_grads} ensures that gradients of the objective function are limited, thereby remain within a certain proximity to the current point.
In Lemma \ref{lem:sufficient_decrease}, we establish a sufficient decrease property, ensuring that the objective function decreases at each iteration by a certain amount.
By combining these lemmas, we demonstrate the convergence of each client in the algorithm, enabling the global convergence proof in Thm.~\ref{thm:global_convergence}.
In summary, our sub goals are as follows:
\begin{enumerate}
    \item We aim to demonstrate the convergence of each client.
    \item We establish that the decrease between client iterations is sufficiently large.
    \item To achieve this, we initially bound all subdifferentials for each client-block, as outlined in Lem.~\ref{lem:bounded_grads}.
    \item Subsequently, we utilize this information to bound the gain.
\end{enumerate}

\begin{lemma}[Convergence of client $i$ in Alg.~\ref{alg:feddc}] 
    Let $\{z^t_i \defeq (U^t_i,V^t)_i\}_{k\in\NN}$ be the sequence generated by a client $i$ in Alg.~\ref{alg:feddc},
    then
    \begin{enumerate}
        \item the client objective $\{\Phi_i(z^t_i)\}_{k}$ converges to $\Phi_i^\ast$, and
        \item the sequence $\{z^t_i\}_k$ converges to a critical point of $\Phi_i(z^\ast_i)$,
    \end{enumerate}
    for $t \to \infty$, assuming that $\Phi_i$ is continuous on $\dom \Phi_i$.
    Furthermore, if a \emph{subsequence} $z^t_i$ starts from the shared coefficients $\Vb$, i.e., $V_i^1 \equiv \Vb$,
    then the difference $\normf{V^t_i - \Vb}$ between $V^t_i$ and $\Vb$ is bounded by a finite constant $\rho$ for $t \to T$.
    \label{thm:local_convergence}
\end{lemma}

Before we proof Lem.~\ref{thm:local_convergence}, we sketch the proof concept as follows.
A problem with block-coordinate methods or Gauss-Seidel approaches lies in showing global convergence for these non-convex problems.
\citet{Attouch:2013:Convergence} demonstrate the convergence of a sequence generated by a generic algorithm to a critical point of a given proper, lower semicontinuous function $\Psi$ (in our case $\Phi_i$) over a Euclidean space $\RR^N$ and establish that the algorithm converges to a critical point of $\Psi$.
Their proof consists of two parts. 
First, they ensure two \emph{fundamental convergence conditions} that are necessary for the convergence of many descent algorithms.
If both are satisfied, they ensure that the set of points of the sequence is nonempty, compact, and connected, with the set being a subset of the critical points of $\Psi$.

\paragraph{Sufficient Decrease Property}
This property ensures that with each iteration, the objective value decreases sufficiently.
Here the aim is to find a positive constant \( \rho_1 \) such that the difference between successive function values decreases sufficiently with each iteration, i.e.,
\[ \rho_1 \left\| z^{t+1} - z^t \right\|^2 \leq \Psi(z^t) - \Psi(z^{t+1}), \quad \forall t = 0, 1, \ldots \]

\paragraph{Subgradient Lower Bound}
This property ensures that the algorithm does not move too far from the current iterate.
Assuming the generated sequence is bounded, we seek another positive constant \( \rho_2 \) such that the norm of the difference between consecutive iterates is bounded by a multiple of the norm of the subgradient of $\Psi$ at the current iterate, i.e.,
\[ \left\| w^{t+1} \right\| \leq \rho_2 \left\| z^{t+1} - z^t \right\|, \quad w^t \in \partial \Psi(z^t), \quad \forall t = 0, 1, \ldots \]
Because we need a certain stability for our Boolean convergence argument, we have to show that we converge to a critical point. 
Second, they show global convergence to a critical point using the KŁ property.

\paragraph{Kurdyka-Łojasiewicz Property}
To establish global convergence to a critical point, they introduce an additional assumption on the class of functions $\Psi$ being minimized, known as the Kurdyka-Łojasiewicz (KŁ) property. 
Intuitively, if this property is satisfied, it prevents the objective to become too flat around a local minimizer, so that the convergence rate would be too low.
It does so by creating a locally-convex/ or simply linear `surrogate' or `gauge' function $g$ that
measures the distance between $z$ and $z^\ast$
\[g(\Psi(z) - \Psi(z^\ast)) \geq \text{dist}(0, \partial \Psi) \]
or more specifically:
\[g(\Psi(z) - \Psi(z^\ast)) \geq \|\partial \Psi\| \]
where, roughly speaking, $z \in \text{Neighborhood}_\eta(z^\ast)$ \cite{Nesterov:2004}.
\citet{Attouch:2013:Convergence} have shown that every bounded sequence generated by the proximal regularized Gauss-Seidel scheme converges to a critical point,
assuming that the objective function satisfies the KL property \cite{Attouch:2013:Convergence}. 
We satisfy this assumption, as our objective is comprised of `semi algebraic' functions.
Now, leveraging the descent property of the algorithm and a uniformization of the KŁ property, they show that the generated sequence is a Cauchy sequence \cite{Attouch:2013:Convergence}, i.e., 
\[ \lim_{l \to \infty} \sum_{t=l}^\infty \left\| z^t - z^{t-1} \right\| \to 0 \;.\]

\begin{proof}(\citet{Attouch:2013:Convergence}).~~%
    Because $\Phi_i$ comprises lower semi-continuous functions on $\dom \Phi_i$, and that all partial gradients are globally Lipschitz,
    all assumptions for the proof are met \cite{Attouch:2013:Convergence}.
    Together with
    (\textbf{\rom{1}}) sufficient decrease property (Lem.~\ref{lem:sufficient_decrease}),
    (\textbf{\rom{2}}) lower-bounded subgradients (Lem.~\ref{lem:bounded_grads}),
    (\textbf{\rom{3}}) the Uniformed K\L~property of (via Lem.~\ref{lem:klp}),
    the convergence lemma follows from the global convergence property in~\citet{Attouch:2013:Convergence}'s proof.
\end{proof}
We now formally proof the three properties, i.e., lower-bounded subgradients, sufficient decrease property, and the uniformed K\L~property.
\begin{lemma}[Lower-bounded Subgradients]
    There is a $\rho$, such that the gradients
    \label{lem:bounded_grads}
    \[
        \operatorname{dist}(0, \partial\Phi_i(z_i^{t+1}))
        \leq \rho \normf{z_i^{t+1} - z_i^t}
    \]
\end{lemma}

\begin{proof}
    To show that the lemma holds, it suffices that we bound each subgradient in the set $\partial\Phi_i(z_{k+1})$ separately.
    Focussing on the $V_i$-block, we want to show
    \[
        \normf{w^{t+1}} \leq \rho_2 \normf{V_i^{t+1}-V_i^{t}}
    \]
    for all
    \(
    w^{t+1} \in \partial \Phi_i(V_i^{t+1})
    \)
    restricted to the  $V_i^{t+1}$-block (analogously repeating the below for $U_i^{t+1}$).
    Because the subdifferential of the maximum-term $\max \{ \overset{\rm\Rom{1}}{r}(x), \overset{\rm\Rom{2}}{r}(x) \}$ is the union of the subdifferentials of its active parts,
    and our regularizer is piecewise convex, we obtain three gradients per block:
    \[
        \partial \Phi_i(U_i, V_i) = \nabla_{V_i} \frac{1}{2}\normf{A_i - U_iV_i}^2 + \nabla_{V_i} \frac{1}{2}\normf{V_i - V^t_i}^2 + \partial R(V_i)\;,
    \]
    \[
        \partial R(V_i) =
        \begin{cases}
            \nabla_{V_i} \overset{\rm\Rom{1}}{r}(V_i)                                               & R(V_i) < R(V_i-1) \\
            \conv(\nabla_{V_i} \overset{\rm\Rom{2}}{r}(V_i), \nabla_V \overset{\rm\Rom{1}}{r}(V_i)) & R(V_i) = R(V_i-1) \\
            \nabla_{V_i} \overset{\rm\Rom{2}}{r}(V_i)                                               & R(V_i) > R(V_i-1) \\
        \end{cases} \;.
    \]
    Next, we bound the norm of the first subdifferential
    \begin{align*}
             & \normf{\nabla_{V} \Phi_i + \partial R(V) + \frac{\gamma}{2} (V - V^t_i)}        \\
        \leq & \normf{\nabla_{V} \Phi_i + \partial R(V_i)} + \frac{\gamma}{2}\normf{V - V^t_i} \\
        \leq & \frac{\rho}{2} \normf{V - V^t_i} + \frac{\gamma}{2}\normf{V - V^t_i}            \\
        \leq & \max\{\rho,\gamma\} \frac{1}{2} \normf{V - V^t_i} \;.
    \end{align*}
    Repeating for the other cases, the total bound $\rho$ is the maximum per block and per subdifferential bounds.
    Based on Lem.~\ref{lem:sufficient_decrease},
    under the assumption that $\Phi_i$ is continuous on its domain,
    and provided that there exists a convergent subsequence (i.e., condition (a)),
    the continuity condition required in \cite{Attouch:2013:Convergence} holds,
    i.e., there exists a subsequence $\{z^t_{i}\}_{k\in \NN}$ and a point $z_i^{\ast}$ such that
    \[ 
        z^t_{i} \to z^{\ast}_i \quad \text{and} \quad \Phi_i(z^t_{i}) \to \Phi_i(z^{\ast})\ \text{as}\ t \to \infty\;. 
    \]
\end{proof}

\begin{lemma}[Sufficient Decrease Property]
    \label{lem:sufficient_decrease}
    For the sequence of points $\{z^t\}_k$ generated by the block-coordinate method in Alg.~\ref{alg:feddc},
    then
    \[
        \Phi_i(z_i^{t+1}) \leq \Phi_i(z_i^{t}) - \rho_1 \normfs{z_i^{t+1} - z_i^{t}}\;.
    \]
\end{lemma}

\begin{proof}
    The loss function for the $V_i$-block in our local block-coordinate descent is
    \[
        \normfs{A_i - U_i^{t}V_i^{t+1}} + \normfs{V_i^{t+1} - \Vb} + R(V_i^{t+1})\;.
    \]
    Likewise for the $U_i$-block
    \[
        \normfs{A_i - U_i^{t+1}V_i^{t+1}} + R(U_i^{t+1})\;.
    \]
    After taking a gradient step, Alg.~\ref{alg:feddc} proceeds with a \emph{Boolean projection} regarding $R$ (for $U$ and $V$ blocks) and a \emph{proximity projection} to $\Vb$ (only for $V$).

    We proceed with the $V_i$-block, while the proof for the $U_i$-block is analogous.
    First, the \textbf{Boolean proximal} projection operator $\prox_{R}(V^{t})$ yields \emph{a minimizer} of the optimization problem
    \[ %
        \ovex{\scriptscriptstyle\Rom{1}}{V}{k} \gets \argmin_Y \sfrac12 \normfs{V^{t} - Y} + R(Y) \;.
    \]
    By definition, $\ovex{\scriptscriptstyle\Rom{1}}{V}{k}_i$ lies in a $\rho_{\Rom{1}}$-bounded proximity to $V_i^{t}$.
    Second, the \textbf{proximity proximal} projection operator $\prox_{\gamma\Vb}(\ovex{\scriptscriptstyle\Rom{1}}{V}{k}_i)$ is the minimizer of
    \[
        \ovex{\scriptscriptstyle\rom{2}}{V}{k}_i \gets \argmin_Y \sfrac12 \normfs{\ovex{\scriptscriptstyle\Rom{1}}{V}{k}_i - Y} + \nu \sfrac12 \normfs{\Vb - Y} \;.
    \]
    By definition, $\ovex{\scriptscriptstyle\rom{2}}{V}{k}_i$ lies in the $\rho_{\Rom{2}}$-bounded proximity to $\ovex{\scriptscriptstyle\Rom{1}}{V}{k}_i$.
    Repeating for the $U_i$-blocks and using a transitivity argument, by using that our gradients have finite Lipschitz moduli, we conclude that both projections lie in a $\rho$-bounded region around $z_i^t$.

    Using the following relationships,
    \begin{align*}
        \Phi_i(z_i^{t+1})                                  & \leq \Phi_i(z_i^{t}) + \rho\normfs{z_i^{t+1}-z_i^{t}}\,,
        \Phi_i(z_i^{t+1}) - \rho\normfs{z_i^{t+1}-z_i^{t}} & \leq \Phi_i(z_i^{t+1})\ \text{, and}
        \Phi_i(z_i^{t+1})                                  & \leq \Phi_i(z_i^{t}).
    \end{align*}
    we now bound the loss reduction in terms of the norm of differences in the following.
    \begin{align*}
        \Phi_i(z_i^{t+1})                                 & \leq \Phi_i(z_i^{t})                                               \\
        \Phi_i(z_i^{t+1}) - \rho\normfs{z^{t+1}-z^{t}}     & \leq \Phi_i(z_i^{t})                                               \\
        \Phi_i(z_i^{t+1}) - \rho\normfs{z_i^{t+1}-z_i^{t}} & \leq \Phi_i(z_i^{t-1}) + \rho\normfs{z_i^{t}-z_i^{t-1}}             \\
        \Phi_i(z_i^{t+1}) - \rho\normfs{z^{t+1}-z^{t}}     & \leq \Phi_i(z_i^{t}) + \rho\normfs{z_i^{t}-z^{t-1}}                 \\
        \Phi_i(z_i^{t+1}) - \Phi_i(z_i^{t})               & \leq \rho\normfs{z_i^{t}-z_i^{t-1}} + \rho\normfs{z_i^{t+1}-z_i^{t}} \\
        \Phi_i(z_i^{t+1}) - \Phi_i(z_i^{t})               & \leq \rho\normfs{z_i^{t+1}-z_i^{t}}
    \end{align*}
\end{proof}

\begin{lemma}[Uniformized Kurdyka-Łojasiewicz (KŁ)]
    \label{lem:klp}
    $\Phi_i$ is a KŁ function.
\end{lemma}
\begin{proof}
    $\Phi_i$ function is composed of $p$-norms ($p\in\{1,2\}$), and indicator functions,
    and therefore satisfy the KŁ-property~\cite{Attouch:2013:Convergence}. %
\end{proof}

\section{Competitors}\label{sec:apx:competitors}
For a given aggregation function (such as rounded averaging~\eqref{eq:aggmean}, majority voting~\eqref{eq:aggvote}, or logical \texttt{or}~\eqref{eq:aggor}), 
we summarize the federation strategy of centralized BMF algorithms in Alg.~\ref{alg:aggregated_bmf}.

\SetKwFor{local}{Locally at client}{do}{}
\SetKwFor{coord}{Centrally at server}{do}{}
\begin{algorithm}[hb]
    \caption{Aggregated BMF}
    \label{alg:aggregated_bmf}
    \KwIn{$C$ clients with local matrices $A_1,\dots,A_C$, local BMF algorithm $\mathcal{A}$, aggregation function $\operatorname{aggregate}$}
    \KwOut{ local feature matrices $U_1,\dots,U_C$, global coefficient matrix $\widehat{V}$}
    \local{$i$}{
        $U_i, V_i \gets \mathcal{A}(A_i)$\;
    }
    \coord{}{
        receive $V_1,\dots,V_C$\;
        $\widehat{V} \gets \operatorname{aggregate}(V_1,\dots,V_C)$\;
        transmit $\widehat{V}$ to all clients\;
    }
    \local{$i$}{
        receive $\widehat{V}$ from the server \; 
        assign $V_i \gets \widehat{V}$\;
    }
\end{algorithm}

\subsection{Obtaining Boolean Matrices from \zhang's Factorization}
The relaxation-based binary matrix factorization of \zhang \citep{Zhang:2007:Binary} does not necessarily yield Boolean factors upon convergence. 
Furthermore,  
this method yields matrices that do not lend themselves to rounding, 
such that in practice, rounding does not yield desirable results \emph{unless} the rounding threshold is carefully chosen.
To choose well-factorizing rounding thresholds, 
we take inspiration from \textsc{Primp}~\citep{Hess:2017:Primping}, searching those thresholds that minimize the reconstruction loss, 
\[ \sum_{c \in [C]} \|A^c - [U_{ij}^c \geq \alpha]_{ij} \circ [V_{ij}^c \geq \beta]_{ij}\| \;,\]
from the equi-distant grid between \num{1e-12} and 1 containing 100 points in each direction.

\section{Datasets}\label{sec:datasets}
To explore the realm of \textbf{recommendation systems}, we have included \emph{Goodreads}~\citep{Kotkov:2022:Goodreads} for book recommendations, as well as \emph{Movielens}~\citep{Harper:2016:Movielens} and \emph{Netflix}~\citep{NetflixPrize} for movie recommendations. 
To focus on positive ratings, we binarized user ratings, setting ratings $\geq 3.5$ to $1$.

In the field of \textbf{life sciences}, we consider cancer genomics through \emph{TCGA}~\citep{TCGA} and single-cell proteomics using 
\emph{HPA}~\citep{Bakken:2021:ComparativeCA,Sjöstedt:2020:hpa}. 
Specifically, \emph{TCGA} records $1$s for gene expressions in the upper 95\% quantile and \emph{HPA} records by $1$ if RNA has been observed in single cells.

For \textbf{social science} inquiries, we investigate poverty \emph{(P)} and income \emph{(I}) analysis using the \emph{Census}~\citep{Census:2023} dataset.
To binarize, we employ one-hot encoding based on the features recommended by Folktables~\citep{ding:2021:retiring}.

In the domain of \textbf{natural language processing}, we focus on higher-order word co-occurrences using ArXiv abstracts from the cs.LG category~\citep{Arxiv:2023}. 
Each paper corresponds to a row whose columns are $1$ if the corresponding word in our vocabulary has been used in its abstract.
The vocabulary consists of words with a minimum frequency of $1$ \textpertenthousand\ in ArXiv cs.LG abstracts (\emph{cs.LG~R}) and their lemmatized, stop-word-free counterparts (\emph{cs.LG}).

We summarize extents, density, and chosen component counts for each real-world dataset in Appendix~\ref{sec:reproducibility}, Table~\ref{tbl:data}.

\begin{table}[t]
    \centering\small
    \caption{Real-world datasets from 4 diverse domains. 
        We show extents, density, and the selected number of components for $10$ real-world datasets.}
    \vspace*{0.5\baselineskip}
    \label{tbl:data}
    \include{figures/tab_datasets.tex}
\end{table}

\section{Reproducibility}
\label{sec:reproducibility}
Supplementing the information provided in Sec.~\ref{sec:experiments},
here, we provide hyperparameter choices for \felb and \felbmu.
We use the iPALM optimization approach for \felb and \felbmu.
Because both algorithms exhibited relatively stable performance fluctuations when it came to tuning, we used the same set of hyperparameters for each experiment and each dataset, thus omitting the commonly necessary hyperparameter tuning step.
In all experiments with\felb and \felbmu, we used 
the regularizer coefficients $\lambda = 0.1$ and $\kappa = 0.001$, 
a regularization rate $\lambda_t = \lambda \cdot 1.05^t$,
an iPALM inertial parameter $\beta = 0.001$,
a maximum number of iterations of $100$,
and a number of local rounds per iteration of $1$, $10$, or $50$, as indicated by the experiments.
For \elb, we choose $\kappa = 0.01$,  $\lambda = 0.01$, $\lambda_t = \lambda \cdot 1.02^t$, and $\beta = 0.01$.
We provide \zhang and \elb with a larger iteration limit of $1\ 000$, multiplying \felb's local rounds by its iteration count.
For \asso, we set gain, loss, and threshold parameters to $1.0$.
For \mebf, we use a threshold of $0.5$ and a cover limit of $0.95$. 

\section{Additional Experiments}
\label{sec:apx:experiments}
Complementing the discussion in Sec.~\ref{sec:experiments},
here, we show additional results for
\asso, \grecond, \mebf, \elb, and \zhang, as well as \felb and \felbmu, for all experiments.
We focus on the quantification not present in the main body of this paper.
Here, we aim to answer the following additional questions.
\begin{itemize}[noitemsep,topsep=0pt,parsep=0pt,partopsep=0pt]
    \item[\bf Q4] How does client drift impact real-world performance?
    \item[\bf Q5] How different the post-hoc aggregations for BMF are?
    \item[\bf Q6] How stably does our methods converge?
    \item[\bf Q7] How robust do we handle client drift?
    \item[\bf Q8] How achievable is differential privacy in different circumstances?
\end{itemize}

\subsection{Real-world Experiments}\label{sec:apx:exp:realworld_rmsd}
In addition to results presented in Table~\ref{tab:realworld}, 
we provide the RMSD in Table~\ref{tab:apx:realworld:rmsd},
where we see that the \felbmu and \felb are the two best performing methods, followed by \grecond.

\begin{table}[h]
    \caption{\felb and \felbmu consistently achieve top performances.
    We illustrate the RMSD of \asso, \grecond, \mebf, \elb, and \zhang under voting aggregation, as well as federated \felb, and \felbmu
    on $8$ real-world data across $50$ clients.
    We highlight the best algorithm with \textbf{bold}, the second best with \underline{underline}, and indicate missing data by a dash {}--{}.}
    \label{tab:apx:realworld:rmsd}
    \centering
    \small
    \vspace*{0.5\baselineskip}
    \input{figures/115/rmsd.tbl.tex}

\end{table}

\subsection{Real-world Drift Experiments}\label{sec:apx:exp:realworld}
Next, because the performance depends on the communication frequency,
we evaluate our method in $3$ different scenarios: 
Rare (max $50$ client epochs), Occasional (max $10$ epochs), and Frequent synchronizations (every round).
To visualize relative performance differences, we compute the \emph{relative RMSD} 
\[
    \frac{\textsc{rmsd}(\textsc{Felb})\hphantom{^{\textsc{mu}}}}{\textsc{rmsd}(\textsc{Felb}^{\textsc{mu}})}\;,
\]
depicted in Fig.~\ref{fig:apx:rwcomm} for all real-world datasets in different synchronization regimes.
Because \asso, \grecond, \mebf, \zhang, and \elbmf are not directly federated, they are independent of the change in communication frequency and therefore omitted.
In Fig.~\ref{fig:apx:rwcomm}, we see that our algorithm maintain a high prediction performance regardless of the communication overhead. 
We observe that \felb and \felbmu perform similarly well under occasional and frequent communications.  
We observe a shrinking performance gap between \felb and \felbmu when increasing the communication frequency, almost reaching the same performance.
This indicates that \felb's larger gradient-step-sizes are responsible for a higher client drift, which is mitigated by a high communication frequency.
Regardless of being under occasional and frequent communication regime, \felb and \felbmu are the highest performing algorithms. 

\begin{figure}[htb]
    \centering
    \includegraphics{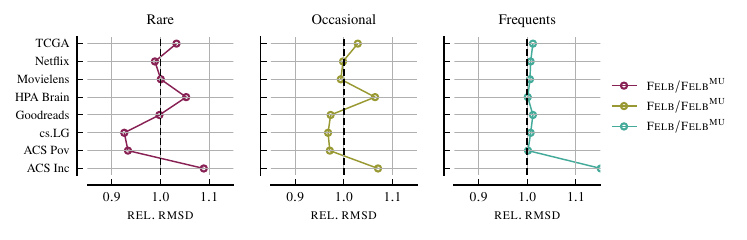}
    \caption{\felb and \felbmu perform similarly when we synchronize clients frequently, 
        while \felbmu tends to improve over \felb if we rarely synchronize. 
        We show the relative RMSD  on real-world datasets with varying communication frequencies for \felb and \felbmu.}
    \label{fig:apx:rwcomm}
\end{figure}

\subsection{Post-hoc Aggregations}\label{sec:apx:aggregations}
As there is no prior art specifically for aggregation federated BMF clients,
we seek experimentally answer which of the equations Eqs.~\eqref{eq:aggmean}--\eqref{eq:aggvote}
yield the lowest reconstruction loss.
To this end, we consider a growing number of synthetic abundant data as described for Fig.~\ref{fig:scalability}.
While we observe in Fig.~\ref{fig:apx:agg} and in Fig.~\ref{fig:apx:agg2} that \emph{rounded average} and \emph{consensus voting} are performing similarly, both significantly outperform \emph{logical or}.
For brevity, we therefore mostly report results for \emph{consensus voting} in Sec.~\ref{sec:experiments}.
\begin{figure}[htb]
    \centering
    \includegraphics{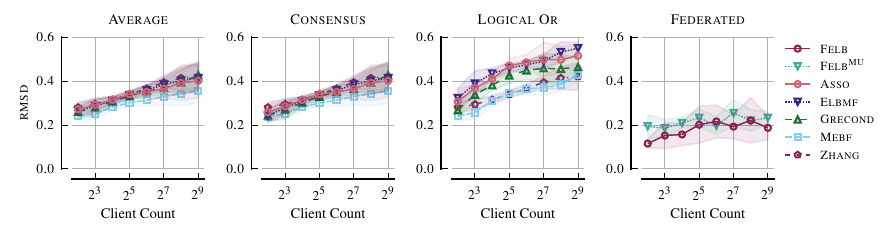}
    \caption{The Boolean matrix aggregation methods \emph{rounded average} and \emph{consensus voting} significantly outperform \emph{logical or}.
        We show the loss for post-hoc aggregated BMF methods, for growing client count with synthetic abundant data.
    }
    \label{fig:apx:agg}
\end{figure}

\begin{figure}[ht]
    \centering
    \includegraphics{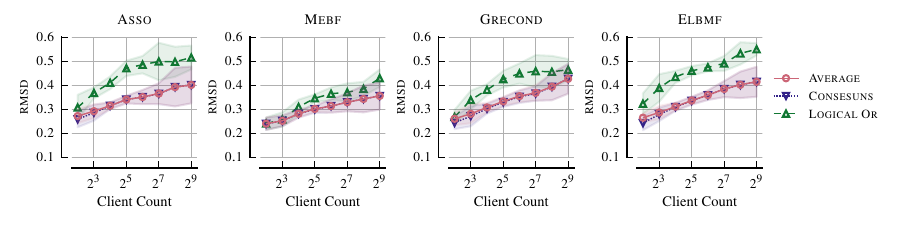}
    \caption{The Boolean matrix aggregation methods \emph{rounded average} and \emph{consensus voting} significantly outperform \emph{logical or}, depicting results specifically for post-hoc aggregated BMF methods, for growing client count with synthetic abundant data.
    }
    \label{fig:apx:agg2}
\end{figure}

\subsection{Empirical Convergence}\label{sec:apx:exp:convergece}
This study aims to investigate the empirical convergence properties of the proposed methods. 
In this study, we examine the empirical convergence properties of our methods. 
We generate synthetic data according to the procedure outlined in Sec.~\ref{sec:experiments}. 
We then measure the reconstruction loss as the number of global iteration steps increases. 
Fig.~\ref{fig:apx:convergence} demonstrates that our methods rapidly converge to a lower loss corresponding to non-Boolean solutions. 
Following a swift initial decrease, the loss only minimally increases as we approach a feasible Boolean solution upon convergence.
\begin{figure}[ht]
    \centering
    \includegraphics{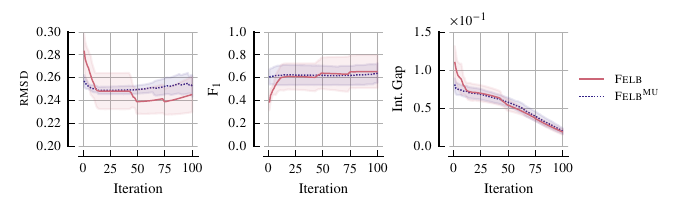}
    \caption{Our methods rapidly achieve a lower reconstruction loss for non-Boolean solutions and maintain minimal loss increase while approaching a feasible Boolean solution. 
    We illustrate the history of loss, F$_1$ score, and integrality gap over increasing number of iterations.}
    \label{fig:apx:convergence}
\end{figure}

\subsection{Client Drift}\label{sec:apx:exp:drift}
We aim to understand the impact of infrequent synchronizations on the convergence results. To investigate this, we vary the number of local iterations per client from 1 (frequent synchronizations) to 50 (infrequent synchronizations), using synthetic data. In Fig.~\ref{fig:apx:drift}, we observe that the loss is significantly affected by the increasing number of iterations. We see that the loss flattens-out after approximately 25 client local optimization epochs before synchronization. 
While our methods achieve a reasonably high F$^\ast_1$-score with respect to the ground-truth---even with infrequent synchronizations---our competitors do not show similar results.

\begin{figure}[ht]
    \centering
    \includegraphics{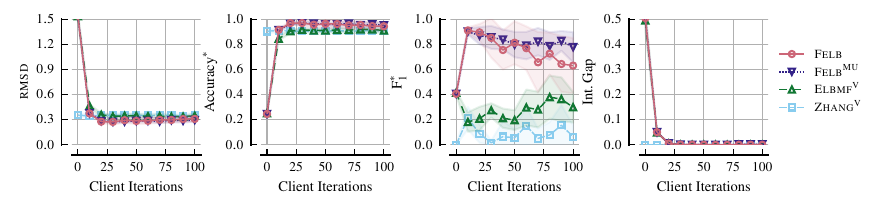}
    \caption{Our algorithm demonstrates robustness in achieving high convergence rates despite infrequent synchronizations. 
    We illustrate the history of loss, F$_1$ score, F$^\ast_1$ score regarding ground-truth, and integrality gap over increasing number of local per-client iterations before global synchronizations.}
    \label{fig:apx:drift}
\end{figure}

\subsection{Differential Privacy}\label{sec:apx:exp:privacy}
We aim to understand how differential privacy impacts reconstruction quality. 
Previously, we studied the effect of clipped noise mechanisms (Fig.~\ref{fig:dp}). 
Here, we extend this experiment to include non-clipped noise mechanisms, as shown in Figures \ref{fig:apx:dp1} and \ref{fig:apx:dp2}. 
Specifically, we apply non-clipped Gaussian and Laplacian noise to federated factorization algorithms that operate on real-valued numbers, while limiting discrete Boolean factorization algorithms to Bernoulli noise.

In Fig. \ref{fig:apx:dp1}, we observe that the F$_1$-score decrease significantly only at high differential privacy coefficients. At moderate levels, we achieve differentially private reconstructions using both clipped and non-clipped Gaussian and Laplacian noise mechanisms, as well as Bernoulli `XOR' noise.
In Fig. \ref{fig:apx:dp2}, we see that the reconstruction loss follows a similar trend for both  Gaussian and Laplacian noise mechanisms.
The Bernoulli mechanism, however, results in a much lower reduction in RMSD than in the F$_1$.
Although all methods exhibit a similar trend, \felb and \felbmu demonstrate robustness regarding differential privacy, consistently outperforming competitors in terms of RMSD and F$_1$ score.

\begin{figure}[htb]
    \centering
    \includegraphics{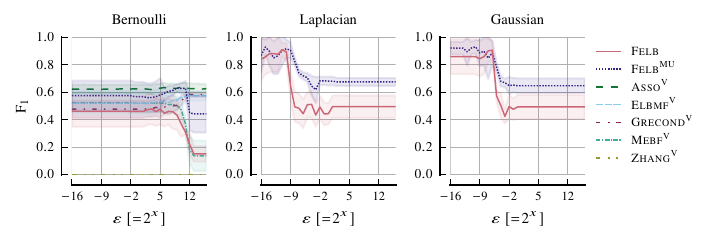}
    \caption{Our algorithms largely maintains the prediction performance for moderately high differential privacy coefficients. We depict the F$_1$-score trend across various levels of differential privacy, for non-clipped Gaussian and Laplacian noise mechanisms, as well as the Bernoulli `XOR' noise mechanism.}
    \label{fig:apx:dp1}
\end{figure}

\begin{figure}[htb]
    \centering
    \includegraphics{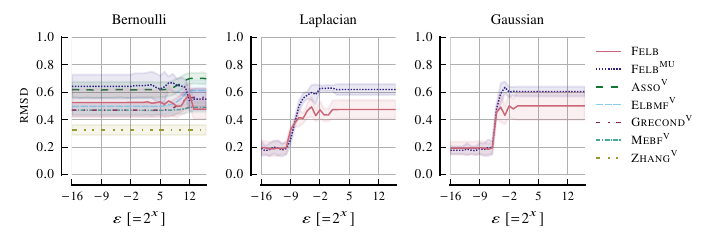}
    \caption{Our algorithms largely maintains the reconstruction quality for moderately high differential privacy coefficients. We depict the reconstruction loss trend across various levels of differential privacy, for  non-clipped Gaussian and Laplacian noise mechanisms, as well as the Bernoulli `XOR' noise mechanism.}
    \label{fig:apx:dp2}
\end{figure}

\section{Limitations} 
\label{sec:limitations}

Our research is motivated by learning from private Boolean data generated by similar sources, situated at few research centers.
As such, we focus on suitable experiments in our research, while we abstain from distant but related problems.

Firstly, our approach does not incorporate personalized federated learning (PFL), which could potentially enhance individual client performance by tailoring the model to specific client data. 
Additionally, our experimental study does not address heterogeneous data distributions across clients, which is a common scenario in real-world applications. 
Furthermore, our focus is on learning and knowledge discovery from federations involving a limited number of clients, specifically in the context of research centers. 
This is in contrast to scenarios involving millions of clients, such as those sometimes encountered in different federated learning applications. 

We experimentally demonstrate under which circumstances our method breaks,
involving experiments with noise levels~\ref{sec:noise}, privacy levels~\ref{sec:exp:dp}, client counts~\ref{sec:scalability}, dataset sizes~\ref{sec:scalability}, client-server communication intervals~\ref{sec:apx:exp:drift}, and dataset domains~\ref{sec:datasets}, thereby providing an extensive overview over practical strength and weaknesses.

\end{document}

%% file: figures/115/f1.tbl.tex
\begin{tabular}{lrrrrrrr}
\toprule
Dataset & {\sc Asso$^\textsc{v}$} & {\sc Mebf$^\textsc{v}$} & {\sc Grecond$^\textsc{v}$} & {\sc Zhang$^\textsc{v}$} & {\sc Elbmf$^\textsc{v}$} & {\sc Felb$^{\textsc{mu}}$} & {\sc Felb} \\
\midrule
ACS Inc & 0.388 & 0.108 & \bfseries 0.690 & 0.000 & 0.000 & \underline{0.585} & 0.328 \\
ACS Pov & \underline{0.692} & -- & \bfseries 0.797 & 0.000 & 0.217 & 0.638 & 0.517 \\
cs.LG & -- & 0.000 & \bfseries 0.068 & 0.000 & 0.000 & \underline{0.057} & 0.006 \\
Goodreads & -- & 0.000 & 0.017 & -- & 0.000 & \bfseries 0.125 & \underline{0.059} \\
HPA Brain & -- & \bfseries 0.642 & -- & 0.000 & 0.000 & \underline{0.007} & 0.000 \\
Movielens & -- & 0.017 & -- & -- & 0.000 & \bfseries 0.193 & \underline{0.163} \\
Netflix & -- & 0.010 & -- & -- & 0.000 & \bfseries 0.197 & \underline{0.144} \\
TCGA & 0.039 & 0.055 & 0.007 & 0.000 & 0.000 & \bfseries 0.414 & \underline{0.402} \\
\midrule\textit{Avg. Rank} & 4.750 & 3.75 & 3.375 & 5.125 & 4.500 & \bfseries 1.625 & \underline{2.750} \\
\bottomrule
\end{tabular}

%% file: figures/tab_datasets.tex
\begin{tabular}{lrrrrr}
\toprule
Dataset & Rows & Cols & Density & Clients & Components \\
\midrule
ACS Inc & 1630167 & 998 & 0.010 & 50 & 20 \\
ACS Pov & 3271346 & 836 & 0.024 & 50 & 20 \\
cs.LG & 145981 & 14570 & 0.005 & 50 & 50 \\
Goodreads & 350332 & 9374 & 0.001 & 50 & 50 \\
HPA Brains & 76533 & 20082 & 0.239 & 50 & 100 \\
Movielens & 162541 & 62423 & 0.002 & 50 & 20 \\
Netflix & 480189 & 17770 & 0.007 & 50 & 20 \\
TCGA & 10459 & 20530 & 0.019 & 50 & 33 \\
\bottomrule
\end{tabular}

%% file: figures/115/rmsd.tbl.tex
\begin{tabular}{lrrrrrrr}
\toprule
Dataset & {\sc Asso$^\textsc{v}$} & {\sc Mebf$^\textsc{v}$} & {\sc Grecond$^\textsc{v}$} & {\sc Zhang$^\textsc{v}$} & {\sc Elbmf$^\textsc{v}$} & {\sc Felb$^{\textsc{mu}}$} & {\sc Felb} \\
\midrule
ACS Inc & 4.583 & 4.929 & \bfseries 3.485 & 5.005 & 5.005 & \underline{3.962} & 4.560 \\
ACS Pov & \underline{5.822} & -- & \bfseries 4.785 & 7.734 & 7.190 & 7.576 & 7.588 \\
cs.LG & -- & 3.398 & \bfseries 3.350 & 3.398 & 3.398 & \underline{3.372} & 3.396 \\
Goodreads & -- & 1.669 & 1.668 & -- & 1.669 & \bfseries 1.641 & \underline{1.660} \\
HPA Brain & -- & \bfseries 19.537 & -- & 24.434 & 24.434 & \underline{24.409} & 24.433 \\
Movielens & -- & 1.956 & -- & -- & 1.962 & \bfseries 1.914 & \underline{1.925} \\
Netflix & -- & 4.075 & -- & -- & 4.084 & \bfseries 3.982 & \underline{4.009} \\
TCGA & 6.858 & 6.834 & 6.871 & 6.872 & 6.872 & \bfseries 6.346 & \underline{6.420} \\
\midrule\textit{Rank} & 4.000 & 3.625 & 2.875 & 5.000 & 4.625 & \bfseries 1.750 & \underline{2.750} \\
\bottomrule
\end{tabular}